\documentclass[11pt, a4paper]{article}
\usepackage[margin=1.25in, top=1.4in, headsep=0.4in]{geometry}
\usepackage{amsmath, amssymb,amsthm,amsfonts,hyperref}
\usepackage{authblk}  
\usepackage[usenames]{xcolor} 
\usepackage{url}
\usepackage{float}
\usepackage{mathrsfs}
\usepackage{enumitem} 
\usepackage{dsfont} 

\usepackage{enumitem}
\setlist{nolistsep}
\usepackage{tikz}
\usetikzlibrary{shapes.misc}
\usetikzlibrary{backgrounds,calc,angles,positioning,intersections,quotes,decorations.markings}
\usetikzlibrary{arrows.meta} 
\usepackage{tkz-euclide}
\usepackage[noend]{algpseudocode}
\usepackage{algorithm}
\usepackage{scalerel,stackengine} 
\usepackage{pgfplots}
\usepackage{pgfplotstable}
\usepackage{subcaption}
\def\l(#1,#2]{\left(#1\,\mathpunct{\ldotp\ldotp}#2\right]} 
\def\o(#1,#2){\left(#1\,\mathpunct{\ldotp\ldotp}#2\right)} 
\def\c[#1,#2]{\left[#1\,\mathpunct{\ldotp\ldotp}#2\right]} 
\def\[#1,#2){\left[#1\,\mathpunct{\ldotp\ldotp}#2\right)}  
\mathchardef\ldotp="613A   
\newcommand{\dd}{\mathbin{\ldotp\!\ldotp}}
\def\:{{\vert\mskip1mu}} 
\def\]{{\downarrow}\mskip1mu} 

\def\(#1){^{\mskip1mu(#1)}}
\def\minus{
  \setbox0=\hbox{-}
  \vcenter{%
    \hrule width\wd0 height \the\fontdimen8\textfont3%
  }%
}
\newcommand{\HY}{H$\mkern.5mu-\mkern-1.5mu$Y}

\newcommand\encircle[1]{%
  \tikz[baseline=(X.base)] 
    \node (X) [draw, shape=circle, inner sep=0] {\strut #1};}

\newcommand\encircled[1]{%
  \tikz[baseline=(X.base)]
    \node (X) [draw, dashed, shape=circle, inner sep=0] {\strut #1};}

\newcommand\encircledd[1]{%
   \tikz[baseline=(X.base)]
    \node (X) [draw, loosely dashdotted, shape=circle, inner sep=0] {\strut #1};}

\newcommand{\bigO}{\mathcal{O}}
\tikzstyle{loosely dashdotted}=      [dash pattern=on 3pt off 4pt on \the\pgflinewidth off 4pt]

\newtheorem{thmx}{Theorem}
\newtheorem{lem}[thmx]{Lemma}
\newtheorem{cor}[thmx]{Corollary}
\newtheorem{remark}[thmx]{Remark}

\newtheorem{prop}{Proposition}

\newcommand{\nonextantP}[1][]{
  \tikz[] \node[myC, #1]{};}
\tikzset{
  o/.style={shape=circle, draw, inner sep=+0pt, minimum size=+2mm},
  myC/.style={o, fill, fake double={o}{+3mm}},
}

\newdimen\unit \unit=1.5em
\def\intpicd #1[#2..#3)#4\\{\hbox{\kern#2\unit\llap{$#1$\ \ }%
 \hbox to0pt{\hss$\bullet$\hss}%
 \dimen0=#3\unit \advance\dimen0-#2\unit
 \hbox to\dimen0{\cleaders\hbox{$-$}\hfil
   }
 \hbox to0pt{\hss$\circ$\hss} \ $#4$}}

\newdimen\unit \unit=1.5em
\def\intpicb #1[#2..#3)#4\\{\hbox{\kern#2\unit\llap{$#1$\ \ }%
 \hbox to0pt{\hss$\bullet$\hss}%
 \dimen0=#3\unit \advance\dimen0-#2\unit
 \hbox to\dimen0{\cleaders\hbox{$\odot$}\hfil
   }
 \hbox to0pt{\hss$\circ$\hss} \ $#4$}}

\newdimen\unit \unit=1.5em
\def\intpicx #1[#2..#3)#4\\{\hbox{\kern#2\unit\llap{$#1$\ \ }%
 \hbox to2.5pt{\hss$\bullet$\hss}%
 \dimen0=#3\unit \advance\dimen0-#2\unit
 \hbox to\dimen0{\cleaders\hbox{$\cdot$}\hfil
   }
 \hbox to2.5pt{\hss$\circ$\hss} \ $#4$}}

\newdimen\unit \unit=1.5em
\def\intpicy #1[#2..#3)#4\\{\hbox{\kern#2\unit\llap{$#1$\ \ }%
 \hbox to0pt{\hss$\bullet$\hss}%
 \dimen0=#3\unit \advance\dimen0-#2\unit
 \hbox to\dimen0{\cleaders\hbox{$\mkern-2mu\smash-\mkern-2mu$}\hfil
   $\mkern-7mu \smash- \mkern1.5mu$}%
 \hbox to0pt{\hss$\circ$\hss} \ $#4$}}

\title{A Note on Asynchronous Challenges: Unveiling Formulaic Bias and Data Loss in the Hayashi$\mkern.1mu-\mkern-1.5mu$Yoshida Estimator}
\author{Evangelos Georgiadis\thanks{Theory Group, CQuant Technologies Limited, The Hong Kong Science \& Technology Park (HKSTP), Hong Kong. Email:\texttt{egeorg\text{@}cquant\text{.}xyz}}}

\date{\today}

\begin{document}
\maketitle
\begin{abstract}

The Hayashi$\,\minus$Yoshida (\HY)-estimator exhibits an intrinsic, telescoping property
that leads to an often overlooked computational bias, which we denote, {\it formulaic or intrinsic bias}. 
This {\it formulaic bias} results in data loss by cancelling out 
potentially relevant data points, the {\it nonextant data points}. This paper attempts to formalize 
and quantify the data loss arising from this bias.
In particular, we highlight the existence of {\it nonextant data points} via a concrete example, 
and prove necessary and sufficient conditions for the telescoping property to induce this 
type of {\it formulaic bias}.

Since this type of bias is nonexistent when inputs, i.e., observation times,
$\Pi^{(1)} :=(t_i^{(1)})_{i=0,1,\ldots}$ and $\Pi^{(2)} :=(t_j^{(2)})_{j=0,1,\ldots}$, are synchronous,
we introduce the $(a,b)$-{\it asynchronous adversary}. This adversary
generates inputs  $\Pi^{(1)}$ and $\Pi^{(2)}$  according to two independent homogenous Poisson processes
with rates $a>0$ and $b>0,$ respectively. We address the foundational questions regarding
cumulative minimal (or least) average data point loss, and determine the values for $a$ and $b$.

We prove that for equal rates $a=b,$ the minimal average cumulative data loss
over both inputs is attained and amounts to $25\%$.
We present an algorithm, which is based on our theorem, for computing the {\it exact number of nonextant data points}
given inputs $\Pi^{(1)}$ and $\Pi^{(2)},$ and suggest alternative methods.
Finally, we use simulated data to empirically compare the (cumulative) average data loss of the (\HY)-estimator.
\end{abstract}

\section{Introduction}\label{intro}


In \cite{HY-2005}, Hayashi and Yoshida proposed an estimator
for the (cumulative) covariance of two diffusion processes
when those are observed at only discrete times in an asynchronous setting. 
More specifically, Hayashi and Yoshida proposed an estimator that natively computes 
on irregular and asynchronous observations.

While the novelty of the Hayashi\,-Yoshida (\HY)-estimator is manifold,
the two key value propositions of their estimator, are, its ability to
\begin{itemize}
\item compute natively on asynchronous data (or observations) --
 avoiding any synchronization, preprocessing that might introduce bias (along with potential data loss) 
\item utilize all available data (or observations)\footnote{assuming no sampling process is employed, 
and all available data points are used.}.
\end{itemize}
\bigskip

Chronologically and conceptually, the (\HY)-estimator represents a third generation statistical
toolkit for analyzing asynchronous high-frequency data,e.g., financial market data.
This estimator, radically, supersedes, its first and second generation predecessors, which
were largely based on conceptually simplistic, if not merely flawed, model assumptions. 
The first generation laid its foundation on the (implicit and very strong) assumptions 
that data was generated by an underlying process exhibiting properties of synchronicity 
and equidistant regularity. The second generation, attempted to rectify these
strong assumptions imposed on real-world data by providing an intermediate 
computational preprocessing layer between raw data and estimator; namely, 
an artificial way to synchronize raw data. Arguably, this attempt turned out to be a 
quick fix solution, introducing a variety of different types of biases, known as 
{\it synchronization bias, imputation bias, etc...}; types which we refer to as 
belonging to the class of {\it extrinsic bias}.\footnote{i.e., bias induced via preprocessing
methods that previous generations of estimators suffered from.}

The first attempt towards a third generation statistical estimator that natively computes on
asynchronous data was devised by Jong and Nijman in~\cite{JN-1997}. 
The estimator avoids arbitrary imputation methods and utilizes all available
transaction data to compute consistent covariance estimates. 
Their approach follows a regression based methodology, and can be viewed as a 
generalization of results in~\cite{CHMSW-1983}, which discusses a related 
problem of {\it intervalling effect bias}. A non-regression based approach with a
continuous time setting, unlike~\cite{JN-1997}, was then proposed by
Hayashi and Yoshida in~\cite{HY-2005}, following further investigations
and generalizations in~\cite{HK-2008},~\cite{HY-2008} and~\cite{HY-2011}.
Despite its nice properties, i.e., being consistent,
asymptotically normally distributed, and, in particular, 
absent from extrinsic bias, the \HY\,estimator
suffers from an intrinsic bias,\footnote{i.e., bias directly associated with the evaluation or computation of the estimator.}
which can be attributed to the telescoping property of its
summation formula. This {\it formulaic bias}, under certain
conditions, leads to data points being completely cancelled out
during computation, which we call {\it nonextant data points.}
These are data points that do not impact the output of the estimator.
Or more concretely, the output of the estimator does not depend on the
value of the functions at those points.


The study of {\it nonextant data points} for an estimator is of pivotal significance.
For one, it provides insights into the inherent limitations of an estimator's
ability to compute on inputs, i.e., addressing the foundational question of
whether an estimator can make use of all available data points that it
is provided with, without (intrinsically) discarding data points in way that do not influence
the output -- leading to potential {\it information loss}. This, in turn, opens up
avenues to formalize an alternative information-theoretic metric of an
estimator's efficiency (to obtaining the true value) with respect to its intrinsic (or formulaic) ability to utilize all inputs.
For another, it provides additional insight of an estimator's {\it breakdown point analysis}
along with its {\it robustness properties}.
The {\it breakdown point} of an estimator, ``is, roughly, the smallest amount
of contamination that may cause an estimator to take on arbitrarily large
aberrant values''~\cite{DN-1983}; or more abstractly, ``the natural
notion that quantifies the effect (or influence) of the outliers on its performance''~\cite{DK-2023}.
In fact, understanding how {\it nonextant data points} occur and whether they follow some form of pattern,
or occur within specific intervals, links to the notion of
{\it stochastic breakdown}, as outlined in (Donoho and Huber)~\cite[Section 5.1 on page 178]{DN-1983},
where contamination arranged in certain patterns is not effective at all at disturbing or impacting
the output of an estimator, than, let's say, contamination that is randomly placed among the data.
This illustrates a deficiency, or, mildly put, property, of an estimator
that an adversary could use to compromise (or mislead) the estimator. For instance,
by enforcing a configuration (or situation) where certain data points, that need to be obfuscated,
fall into patterns that do not impact the output of estimator.
In the context of a naive, direct application of the (\text{\HY})-estimator~(\ref{org:def}),
which is used to quantify a lead-lag relationship,
the existence of {\it nonextant data points} highlights
inherent limitations in its ability to quantify or capture granular (lead-lag) variations.
This property, also opens up potential avenues for an adversary to compromise
the estimator, by rendering its output to be misleading. For example, an adversary that wants
to obfuscate certain data points in the presence of a (\text{\HY})-analysis, might enforce
a configuration in which these data points cancel out without impacting the output of the estimator. 
To this end, a thorough and rigorous analysis of this estimator is of interest to
regulators and regulatory institutions, such as the U.S. Securities and Exchange Commission (SEC),
whose goal is to detect and prevent market manipulation. We note that this estimator has 
extensively been deployed as the chosen statistical toolkit to consult the SEC on certain market activities of exchanges,
due to its characteristics of being {\it (supposedly) free from any bias and able to handle asynchronicity},
as seen from recent SEC related filings~\cite[page 5536]{SEC-2022}.
Last but not least, this study provides a reference point to compare
{\it average data loss} arising from classical, extrinsic bias, such as {\it synchronization and imputation bias},
to {\it average data loss} arising from our {\it intrinsic or formulaic bias}.

Hence, this note serves as a natural first step towards understanding this {\it formulaic bias},
which involves analyzing the conditions that induce {\it nonextant data points},
providing an algorithm to compute the {\it exact} number of nonextant data points,
along with addressing the question of estimating the least (cumulative) average data loss
under real world type inputs, i.e.  inputs that are asynchronous.

\subsection{Structure of paper}

We highlight via concrete example an arguably counterintuitive {\it bias} in the
(\HY)-estimator which can be attributed to the intrinsic, telescoping property of its summation formula.
We show that this {\it formulaic bias} induces data points (or observations) 
to be cancelled out completely during the computation, which we denote as {\it nonextant data points}.
Furthermore, we observe that these data points do not influence the output of the estimator.
We note that {\it formulaic bias} exhibits similar characteristics to those 
resulting from {\it synchronization or imputation bias}. 

Second, we characterize {\it nonextant data points}, illustrate and prove the conditions
needed for {\it nonextant data points} to arise. Our theorem is sufficiently conducive
to employ as decision procedure for deciding whether a data point is nonextant.

Third, we derive an expression for the expected number of {\it nonextant data points} assuming that
inputs are generated by our $(a,b)$-asynchronous adversary.

Finally, we present two different algorithms to compute the exact number of {\it nonextant data points}.
We employ a Monte Carlo style exploration to empirically investigate (cumulative) data loss 
as the number of time units for the Poisson processes increases, over varying rates. 
For each instance of varying rates, we generate a fixed number of random samples and compute
the data loss.
The results obtained provide a reference point for future iterations of our work and provide
a basis for comparison to our theoretically obtained results.

\section{Preliminaries: definitions, remarks and examples}\label{build:def}

\subsection{Our version}\label{our:def}

In this section, we are building up definitions to construct expression~\eqref{our:def}, 
the equivalent, original \HY\,formula\footnote{without the sampling mechanism, as we compute on all available data}.
For interval notation, we adopt the minimalized variant of that pioneered by C.~A.~R. Hoare in his note~\cite[p.~337]{H-1972}. 
The original notation had three dots; it was then modified by Lyle~H.~Ramshaw, and popularized in~\cite{GKP-1994}. 

Let $\[a,b)$ be a half-open interval, denoting the set of real numbers $x$ in the range $a \leq x < b$. Similar notations apply to open intervals $\o(a,b)$,
closed intervals $\c[a,b]$ and half-open intervals that include the right endpoint but not the left, $\l(a,b].$ 

We fix $T \in \o(0,\infty)$ to be an arbitrary terminal time for observing\footnote{i.e., we record changes in prices instantaneously.} 
prices of two securities, $P^{(l)}$, where $l=1,2.$ Accordingly, we define, $\Pi^{(1)} :=(t_i^{(1)})_{i=0,1,\ldots,M^{(1)}}$ 
and $\Pi^{(2)}:=(t_j^{(2)})_{j=0,1,\ldots,M^{(2)}},$ to be increasing sequences of random observation times such that 
$t_0^{(1)}, t_0^{(2)} \geq 0$ and $t_{M^{(1)}}^{(1)}, t_{M^{(2)}}^{(2)} \leq T.$ We will sometimes refer to observation times as points.
We note that the estimator cannot natively compute on multiplicities without preprocessing. Additionally, for our work,
we assume that with probability $1$ all points in both $\Pi^{(1)}$ and $\Pi^{(2)}$ are distinct,
by invoking our {\it asynchronous adversary} to generate completely asynchronous inputs.
Additionally, we define the intervals between successive observation times as $I_i^{(1)}:= \left(t_{i-1}^{(1)} \dd t_i^{(1)} \right]$ with 
$I^{(1)} := \bigcup_{i=1}^{\left(M^{(1)}\right)}(I_i^{(1)}),$ and accordingly $I_j^{(2)}:= \left(t_{j-1}^{(2)} \dd t_j^{(2)} \right]$  with 
$I^{(2)} := \bigcup_{j=1}^{\left(M^{(2)}\right)}(I_j^{(2)}).$ 
For convenience, we let $t_0^{(1)}\land t_0^{(2)}:= \min{(t_0^{(1)},t_0^{(2)})}$ and $t_0^{(1)}\lor t_0^{(2)}:= \max{(t_0^{(1)},t_0^{(2)})}.$
Thus the covariance between $P^{(1)}$ and $P^{(2)}$  over period 
$|I^{(1)} \cap I^{(2)}| = (t_0^{(1)} \lor t_0^{(2)} \dd t_{M^{(1)}}^{(1)} \land t_{M^{(2)}}^{(2)}],$ is provided via

\begin{align}
{\left<P_{\Pi^{(1)}}^{(1)}, P_{\Pi^{(2)}}^{(2)}\right>}_{\text{\HY}}&:= \sum_{j=1}^{\left(M^{(2)}\right)} \sum_{i=1}^{\left(M^{(1)}\right)} \left( P_{t_i^{(1)}}^{(1)} -  P_{t_{i-1}^{(1)}}^{(1)} \right) \left(P_{t_j^{(2)}}^{(2)} - P_{t_{j-1}^{(2)}}^{(2)}\right)\mathds{1}_{\left\{ \l(t_{i-1}^{(1)}, t_{i}^{(1)}] \cap \l(t_{j-1}^{(2)},t_{j}^{(2)}]\neq\emptyset\right\}} \nonumber \\
&\;= \sum_{i,j} \left(\Delta P_i^{(1)} \right) \left( \Delta P_j^{(2)} \right)\mathds{1}_{\{ I_i^{(1)} \cap I_j^{(2)}\neq\emptyset\}},\label{our:def:form}
\end{align}

where $\Delta P_i^{(1)}:=  P_{t_{i}^{(1)}}^{(1)} - P_{t_{i-1}^{(1)}}^{(1)}$ and $ \Delta P_j^{(2)} := P_{t_{j}^{(2)}}^{(2)} - P_{t_{j-1}^{(2)}}^{(2)}.$ 

\subsection{Original version}\label{org:def}

The original definition of the (cumulative) covariance estimator is
presented in~\cite[Definition~3.1, on page 368]{HY-2005} and outlined
for completeness and convenience below. Note that our expression in (\ref{our:def:form}) is
equivalent to (\ref{eqn:original}).
\begin{equation}\label{eqn:original}
\sum_{i,j}\Delta P^{(1)}(I^i)\Delta P^{(2)}(J^j)\mathds{1}_{\{I^i \cap J^j\neq\emptyset\}},
\end{equation}
 
where, $P^{(1)}(I^i)$ is essentially our $P_i^{(1)}$, and $P^{(2)}(J^j)$  is essentially our $P_j^{(2)}$; 
additionally, we use $I_i^{(1)}$ and $I_j^{(2)}$ in place of $I^i$ and $J^j,$ respectively.

\begin{remark}
At first inspection, \eqref{our:def:form} reveals that 
contributions to the sum depend on the indicator function. 
That is, the only contributions of the product of any pair of
(successive) differences $\Delta P_i^{(1)} \Delta P_j^{(2)}$
to the sum occur only when the respective observation intervals
$I_i^{(1)}$ and $I_j^{(2)}$ overlap. 

A deeper inspection of \eqref{our:def:form} reveals that due to
the telescoping nature of the summation not all of the 
{\it observations or data points} actually end up counting towards the sum.
Cancellation enters the computation, resulting in observations or data points
being {\it intrinsically cancelled out}. The cancelled out data points are our 
{\it nonextant data points}. To illustrate instances that generate {\it nonextant data points},
we provide the following concrete example below~\ref{example:concrete}.
\end{remark}
\begin{remark}
Further, a word on interval type. The \HY\, formula is based on {\it half-open} intervals
that include right endpoints, in particular, 
$\mathds{1}_{\left\{ \l(t_{i-1}^{(1)}, t_{i}^{(1)} ] \bigcap \l(t_{j-1}^{(2)}, t_{j}^{(2)}]\neq\emptyset\right\}}.$  
Since the phenomenon of cancellation and thus nonextant data points is invariant to whether 
intervals are half-open to the right or left, and since half-open intervals that include
the left endpoints are slightly more common and slightly more convenient to compute with, we could conduct 
our analysis and computation on half-open intervals that include the left endpoint.
\end{remark}

\subsection{Concrete example}\label{example:concrete}

Suppose two securities have the following observation times (in units of time,e.g., seconds) and prices (in USD denomination).
We let terminal time, $T = 12,$ and have observation times $\Pi^{(1)} = (2,3,4,5,7,8,11.5),$ with respectively observed prices 
$P^{(1)} = (10,15,25,10,5,1,5),$ and $\Pi^{(2)} = (1,6,9,10,11,12),$ with respectively observed prices $P^{(2)} = (5,10,15,20,25,20).$ 
To put our notation into use, the third\footnote{Note, the index starts at $0$.} observation time of $P^{(1)}$ is $t_2^{(1)} = 4$, and thus, 
the price of $P^{(1)}$ at the third observation time is $P_{t_2^{(1)}}^{(1)} = 25.$
The third\footnote{Note, the index starts at $1$.} observation interval of $P^{(1)}$ is $I_3^{(1)}.$
Hence,  $I_3^{(1)} = \l(t_{2}^{(1)},t_{3}^{(1)}] = \l(4,5].$
The following figure illustrates observation times $\Pi^{(1)}$ and $\Pi^{(2)}.$


\tikzset{cross/.style={cross out, draw=black, minimum size=2*(#1-\pgflinewidth), inner sep=0pt, outer sep=0pt},
cross/.default={1pt}}
\begin{figure}[!htp]
\begin{center}
\hspace{-.75cm}
\begin{tikzpicture}
\draw   (0,1) -- (12,1) 
	(0,.9) -- (0.,1.1) 
	(12,.9) -- (12,1.1) 
	(0,-1) -- (12,-1) 
	(0, -0.9) -- (0,-1.1) 
	(12,-0.9) -- (12,-1.1); 
\draw[dashed,blue] (2.5,1) -- (2.5,0.4);
\draw[dashed,blue] (2.5,-0.2) --(2.5,-.5) -- (3.5,-1);
\draw[dashed,blue] (3.515,1) -- (3.515,0.4);
\draw[dashed,blue] (3.515,-0.2) --(3.515,-.5) -- (3.5,-1);
\draw[dashed,blue] (4.515,1) -- (4.515,0.4);
\draw[dashed,blue] (4.515,-0.2) --(4.515,-.5) -- (3.5,-1);
\draw[dashed,blue] (6.015,1) -- (5.515,.8) -- (5.515,0.4);
\draw[dashed,blue] (5.515,-0.2) --(5.515,-.5) -- (3.5,-1);
\draw[dotted,blue] (6.015,1) -- (6.515,.8) -- (6.515,0.4);
\draw[dotted,blue] (6.515,-0.2) --(6.515,-.5) -- (7.5,-1);
\draw[dotted,blue] (7.515,1) -- (7.515,0.4);
\draw[dotted,blue] (7.515,-0.2) --(7.515,-.5) -- (7.5,-1);
\draw[dotted,blue] (9.7515,1) -- (8.515,.8) -- (8.515,0.4);
\draw[dotted,blue] (8.515,-0.2) --(8.515,-.5) -- (7.5,-1);
\draw[densely dashdotted,blue] (9.7515,1) --(9.515,.8) -- (9.515,0.4);
\draw[densely dashdotted,blue] (9.515,-0.2) --(9.515,-.5) -- (9.5,-1);
\draw[densely dotted,blue] (9.7515,1) -- (10.515,0.8) -- (10.515,0.4);
\draw[densely dotted,blue] (10.515,-0.2) --(10.515,-.5) -- (10.5,-1);
\draw[loosely dashdotted,blue] (9.7515,1) -- (11.28,.8) -- (11.28,0.4);
\draw[loosely dashdotted,blue] (11.28,-0.2) --(11.28,-.5) -- (11.5,-1);
\foreach \Point in {(2,1), (3,1), (4,1), (5,1), (7,1), (8,1), (11.5,1)}{
    \node at \Point {\textbullet};}
\foreach \Point in { (1,-1.01), (6,-1.01), (9,-1.01), (10,-1.01), (11,-1.01), (12,-1.01)}{
    \node at \Point {\textbullet};} 
\draw[decorate,decoration={brace,amplitude=4pt}]
	(2,1.01)--(3.,1.01) node[midway, above,yshift=5pt,]{\small{$|I_{1}^{1}|$}};
\draw[decorate,decoration={brace,amplitude=4pt, mirror}]
        (4.01,1.01)--(3.02,1.01) node[midway,above,yshift=5pt,]{\small{$|I_{2}^{1}|$}};
\draw[decorate,decoration={brace,amplitude=4pt, mirror}]
        (5.01,1.01)--(4.02,1.01) node[midway,above,yshift=5pt,]{\small{$|I_{3}^{1}|$}};
\draw[decorate,decoration={brace,amplitude=4pt, mirror}]
        (7.01,1.01)--(5.02,1.01) node[midway,above,yshift=5pt,]{\small{$|I_{4}^{1}|$}};
\draw[decorate,decoration={brace,amplitude=4pt, mirror}]
        (8.01,1.01)--(7.02,1.01) node[midway,above,yshift=5pt,]{\small{$|I_{5}^{1}|$}};
\draw[decorate,decoration={brace,amplitude=4pt, mirror}]
        (11.5,1.01)--(8.02,1.01) node[midway,above,yshift=5pt,]{\small{$|I_{6}^{1}|$}};
\draw[decorate,decoration={brace,amplitude=4pt}]
       (6,-1.6)--(1.,-1.6) node[midway, below,yshift=-5pt,]{\small{$|{I}_{1}^{2}|$}};
\draw[decorate,decoration={brace,amplitude=4pt}]
       (9,-1.6)--(6.02,-1.6) node[midway, below,yshift=-5pt,]{\small{$|{I}_{2}^{2}|$}};
\draw[decorate,decoration={brace,amplitude=4pt}]
       (10,-1.6)--(9.02,-1.6) node[midway, below,yshift=-5pt,]{\small{$|{I}_{3}^{2}|$}};
\draw[decorate,decoration={brace,amplitude=4pt}]
       (11,-1.6)--(10.02,-1.6) node[midway, below,yshift=-5pt,]{\small{$|{I}_{4}^{2}|$}};
\draw[decorate,decoration={brace,amplitude=4pt}]
       (12,-1.6)--(11.02,-1.6) node[midway, below,yshift=-5pt,]{\small{$|{I}_{5}^{2}|$}};
\draw[decorate,decoration={brace,amplitude=4pt}]
	(3,0.45)--(2,0.45) node[midway,below,yshift=-5pt,]{\tiny{$\big|\! {I}_{1}^{1}\!\cap\!{I}_{1}^{2}\!\big|$}};	
\draw[decorate,decoration={brace,amplitude=4pt}]
	(4.0,0.45)--(3.05,0.45) node[midway, below,yshift=-5pt,]{\tiny{$\big|\! {I}_{2}^{1}\!\cap\!{I}_{1}^{2}\!\big|$}};
\draw[decorate,decoration={brace,amplitude=4pt}]
	(5.0,0.45)--(4.05,0.45) node[midway, below,yshift=-5pt,]{\tiny{$\big|\!{I}_{3}^{1}\!\cap\!{I}_{1}^{2}\!\big|$}};
\draw[decorate,decoration={brace,amplitude=4pt}]
	(6,0.45)--(5.05,0.45) node[midway, below,yshift=-5pt,]{\tiny{$\big|\!{I}_{4}^{1}\!\cap\!{I}_{1}^{2}\!\big|$}}; 
\draw[decorate,decoration={brace,amplitude=4pt}]
        (7,0.45)--(6.05,0.45) node[midway, below,yshift=-5pt,]{\tiny{$\big|\!{I}_{4}^{1}\!\cap\!{I}_{2}^{2}\!\big|$}}; 
\draw[decorate,decoration={brace,amplitude=4pt}]
	(8.,0.45)--(7.05,0.45) node[midway, below,yshift=-5pt,]{\tiny{$\big|\!{I}_{5}^{1}\!\cap\!{I}_{2}^{2}\!\big|$}};
\draw[decorate,decoration={brace,amplitude=4pt}]
	(9.,0.45)--(8.05,0.45) node[midway, below,yshift=-5pt,]{\tiny{$\big|\!{I}_{6}^{1}\!\cap\!{I}_{2}^{2}\!\big|$}};
\draw[decorate,decoration={brace,amplitude=4pt}]
        (10.,0.45)--(9.05,0.45) node[midway, below,yshift=-5pt,]{\tiny{$\big|\!{I}_{6}^{1}\!\cap\!{I}_{3}^{2}\!\big|$}};
\draw[decorate,decoration={brace,amplitude=4pt}]
        (11.,0.45)--(10.05,0.45) node[midway, below,yshift=-5pt,xshift=-1pt]{\tiny{$\big|\!{I}_{6}^{1}\!\cap\!{I}_{4}^{2}\!\big|$}};
\draw[decorate,decoration={brace,amplitude=3.5pt}]
        (11.5,0.45)--(11.05,0.45) node[midway, below,yshift=-5pt,xshift=2pt]{\tiny{$\big|\!{I}_{6}^{1}\!\cap\!{I}_{5}^{2}\!\big|$}};
\draw (0.,1) node[left,xshift=-2pt]{$\Pi^{1}$}
(0.,-1) node[left,xshift=-2pt]{$\Pi^{2}$};
\draw (0.,0.25) node[below,yshift=-1pt]{\small $0$} 
(12.,0.25) node[below,yshift=-1pt,xshift=9pt]{\small $T\!\!=\!\!12$}; 
\draw[densely dotted] (0,1) -- (0,.15)
              (0,-1) -- (0,-.2)
              (12,1) -- (12,.15)
              (12,-1) -- (12,-.2);
\draw[densely dashed] (12,1) -- (13,1)
              (12,-1) -- (13,-1);
\foreach \x / \y in {0/{2},1/{3},2/{4},3/{5},4/{7},5/{8},6/{11.5}}
   \node[below, yshift=-.4pt] at (\y,1) {\small $t_{\x}^1$};
\foreach \x / \y in {0/{1},1/{6},2/{9},3/{10},4/{11},5/{12}}
   \node[below, yshift=-.4pt] at (\y,-1) {\small $t_{\x}^2$};
\end{tikzpicture}
\end{center}
\caption{Visualizing notation in use for ${\left<P_{\Pi^{(1)}}^{(1)}, P_{\Pi^{(2)}}^{(2)}\right>}_{{\text{\HY}}}$. Illustrating the relevant interval overlaps where  
$\mathds{1}_{\left\{\l( t_{i-1}^{(1)},t_{i}^{(1)}]\bigcap \l( t_{j-1}^{(2)}, t_{j}^{(2)}]\neq\emptyset\right\}}$ holds. There are ten corresponding interval overlaps.
Note, this visualization is not sufficiently granular for highlighting half-open intervals. 
We further note that the input sequence of observation times is equivalent to $BAAAABAABBBAB,$ assuming points in $\Pi^{(2)}$ have been labelled with $B$,
and points in $\Pi^{(1)}$ have been labelled with an $A.$}
\label{fig:1}
\end{figure}
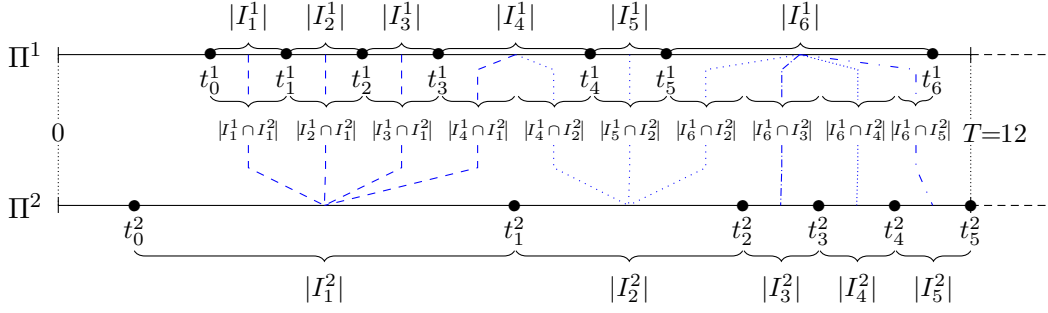
Figure~\ref{fig:1} illustrates that interval $I_1^{(2)}$ overlaps with intervals $I_1^{(1)},I_2^{(1)},I_3^{(1)}$, and $I_4^{(1)},$ respectively;
interval,$I_2^{(2)}$, overlaps with intervals $I_4^{(1)},I_5^{(1)}$ and $I_6^{(1)},$ respectively; and intervals $I_3^{(2)}$ overlaps with interval
$I_6^{(1)}$. Interval, $I_4^{(2)}$, also overlaps with interval $I_6^{(1)}$; and $I_5^{(2)}$ also overlaps with $I_6^{(1)}.$ We could have summarized
the last three overlaps more succinctly,namely, interval $I_6^{(1)}$ also overlaps with intervals $I_3^{(2)},I_4^{(2)}$
and $I_5^{(2)}.$

\begin{align}
{\left<P_{\Pi^{(1)}}^{(1)}, P_{\Pi^{(2)}}^{(2)}\right>}_{{\text{\HY}}}&:= \sum_{j=1}^{5} \sum_{i=1}^{6} \left( P_{t_i^{(1)}}^{(1)} -  P_{t_{i-1}^{(1)}}^{(1)} \right) \left(P_{t_j^{(2)}}^{(2)} - P_{t_{j-1}^{(2)}}^{(2)}\right)\mathds{1}_{\left\{\l(t_{i-1}^{(1)},t_{i}^{(1)} ] \bigcap \l(t_{j-1}^{(2)}, t_{j}^{(2)}]\neq\emptyset\right\}} \label{async1} \\
&\,\text{Evaluating only where $\mathds{1}$ holds,}\nonumber\\
&\,= \left( P_{t_1^{(1)}}^{(1)} -  P_{t_{0}^{(1)}}^{(1)} \right) \left(P_{t_1^{(2)}}^{(2)} - P_{t_{0}^{(2)}}^{(2)}\right) \nonumber \\
&\,+ \left( P_{t_2^{(1)}}^{(1)} -  P_{t_{1}^{(1)}}^{(1)} \right) \left(P_{t_1^{(2)}}^{(2)} - P_{t_{0}^{(2)}}^{(2)}\right) \nonumber \\
&\,+ \left( P_{t_3^{(1)}}^{(1)} -  P_{t_{2}^{(1)}}^{(1)} \right) \left(P_{t_1^{(2)}}^{(2)} - P_{t_{0}^{(2)}}^{(2)}\right) \nonumber \\
&\,+ \left( P_{t_4^{(1)}}^{(1)} -  P_{t_{3}^{(1)}}^{(1)} \right) \left(P_{t_1^{(2)}}^{(2)} - P_{t_{0}^{(2)}}^{(2)}\right) \nonumber \\
&\,+ \left( P_{t_4^{(1)}}^{(1)} -  P_{t_{3}^{(1)}}^{(1)} \right) \left(P_{t_2^{(2)}}^{(2)} - P_{t_{1}^{(2)}}^{(2)}\right) \nonumber \\
&\,+ \left( P_{t_5^{(1)}}^{(1)} -  P_{t_{4}^{(1)}}^{(1)} \right) \left(P_{t_2^{(2)}}^{(2)} - P_{t_{1}^{(2)}}^{(2)}\right) \nonumber \\
&\,+ \left( P_{t_6^{(1)}}^{(1)} -  P_{t_{5}^{(1)}}^{(1)} \right) \left(P_{t_2^{(2)}}^{(2)} - P_{t_{1}^{(2)}}^{(2)}\right) \nonumber \\
&\,+ \left( P_{t_6^{(1)}}^{(1)} -  P_{t_{5}^{(1)}}^{(1)} \right) \left(P_{t_3^{(2)}}^{(2)} - P_{t_{2}^{(2)}}^{(2)}\right) \nonumber \\
&\,+ \left( P_{t_6^{(1)}}^{(1)} -  P_{t_{5}^{(1)}}^{(1)} \right) \left(P_{t_4^{(2)}}^{(2)} - P_{t_{3}^{(2)}}^{(2)}\right) \nonumber \\
&\,+ \left( P_{t_6^{(1)}}^{(1)} -  P_{t_{5}^{(1)}}^{(1)} \right) \left(P_{t_5^{(2)}}^{(2)} - P_{t_{4}^{(2)}}^{(2)}\right) \nonumber \\
&\, \text{Factoring sequentially and pairwise encircling cancellations,} \nonumber\\
&\, \text{we obtain the following.}\label{asyncseqf1} \nonumber \\
\end{align}
\begin{align}
&\, = \Delta P_{1}^{(2)} \left[ \left(\text{\encircle{$P_{t_1^{(1)}}^{(1)}$}} -  P_{t_{0}^{(1)}}^{(1)} \right) + \left(\text{\encircled{$P_{t_2^{(1)}}^{(1)}$}} -  \text{\encircle{$P_{t_{1}^{(1)}}^{(1)}$}}\right) + \left(\text{\encircledd{$P_{t_3^{(1)}}^{(1)}$}} - \text{\encircled{$P_{t_{2}^{(1)}}^{(1)}$}}\right) + \left(P_{t_4^{(1)}}^{(1)} -  \text{\encircledd{$P_{t_{3}^{(1)}}^{(1)}$}}\right)  \right] \nonumber \\
&\, + \Delta P_{2}^{(2)} \left[ \left(\text{\encircle{$P_{t_4^{(1)}}^{(1)}$}} -  P_{t_{3}^{(1)}}^{(1)}\right) + \left( \text{\encircled{$P_{t_5^{(1)}}^{(1)}$}} -  \text{\encircle{$P_{t_{4}^{(1)}}^{(1)}$}}\right) + \left(P_{t_6^{(1)}}^{(1)} -  \text{\encircled{$P_{t_{5}^{(1)}}^{(1)}$}}\right)  \right] \nonumber \\
&\, + \Delta  P_{6}^{(1)} \left[  \left(\text{\encircle{$P_{t_3^{(2)}}^{(2)}$}} - P_{t_{2}^{(2)}}^{(2)}\right) + \left(\text{\encircled{$P_{t_4^{(2)}}^{(2)}$}} - \text{\encircle{$P_{t_{3}^{(2)}}^{(2)}$}}\right) + \left(P_{t_5^{(2)}}^{(2)} - \text{\encircled{$P_{t_{4}^{(2)}}^{(2)}$}}\right)\right] \nonumber \\ 
&\, \text{simplifying after cancellations,} \nonumber \\
&\, =  \left(P_{t_1^{(2)}}^{(2)} - P_{t_{0}^{(2)}}^{(2)}\right) \left( P_{t_4^{(1)}}^{(1)}  - P_{t_{0}^{(1)}}^{(1)}  \right) + \left(P_{t_2^{(2)}}^{(2)} - P_{t_{1}^{(2)}}^{(2)}\right) \left(P_{t_6^{(1)}}^{(1)} - P_{t_{3}^{(1)}}^{(1)}\right) + \left( P_{t_6^{(1)}}^{(1)} -  P_{t_{5}^{(1)}}^{(1)}\right)  \left(  P_{t_{5}^{(2)}}^{(2)} -  P_{t_{2}^{(2)}}^{(2)}\right) \label{old} \nonumber \\
\end{align}

\begin{remark}\label{cm:cancelled}
Ten product summands are reduced to three product summands.
Given the sequence of computation which involves sequential factorization in step  (\ref{asyncseqf1}) of (\ref{async1}), 
four cancellations occur. In particular, data points $P_{t_1^{(1)}}^{(1)}, P_{t_2^{(1)}}^{(1)}, P_{t_3^{(2)}}^{(2)}, P_{t_4^{(2)}}^{(2)}$ are cancelled.
These are the four {\it nonextant} data points.
\end{remark}

\begin{remark}\label{remark:observation:times}
  Since observation times, $\Pi^{(1)}$ and $\Pi^{(2)}$, rather than prices, $P^{(1)}$ and $P^{(2)},$ dictate
  the cancellation phenomenon, for simplicity, some of the analysis will be reduced
  to considering $\Pi^{(1)}$ and $\Pi^{(2)}$ as inputs.
\end{remark} 

\begin{remark}
Finally, the sequence of computations involving sequential factorization needs to be examined more carefully.
In particular, if we factor out $\left( P_{t_4^{(1)}}^{(1)} -  P_{t_{3}^{(1)}}^{(1)} \right),$ we end up with four
(instead of three) product summands. However, we further note, (\ref{new}) is equivalent to (\ref{old}) in value 
but not in minimal representation; for the actual computation, refer to the appendix in (\ref{appendix:1}).
Through the lens of symbolic computation, this illustrates a canonicalization challenge that requires attention and
opens avenues for further exploration with respect to whether a standard, minimal representation exists 
that best exemplifies the maximal cancellation\footnote{If so, what is an efficient procedure 
to compute such a representation. In what context is such a representation unique?}.
\end{remark}

\subsection{(a,b)-Asynchronous Adversary, asynchronous inputs}

Our asynchronous adversary is a statistical adversary that assures that inputs
$\Pi^{(1)}$ and $\Pi^{(2)}$ are generated by two independent homogenous Poisson processes
with rates $a>0$ and $b>0$, respectively, guaranteeing that inputs maintain a certain property.
In our case, this property is that with probability 1, inputs are asynchronous; additionally, our
adversary assures that initial overlaps occur at $I^{(1)}_1$  with $I^{(2)}_1$, and terminal overlaps occur
analogously. This property is necessary for Remark~\ref{remark:points:overlaps} to hold.

Our concrete example in (\ref{example:concrete}) illustrates an instance of asynchronous inputs,
$\Pi^{(1)} = (2,3,4,5,7,8,11.5),$ and  $\Pi^{(2)} = (1,6,9,10,11,12).$ 
On the other hand, an instance of (complete) synchronous inputs would look
like $\Pi^{(1)} = (2,3,4,5,7,8,11.5),$ and $\Pi^{(2)} = (2,3,4,5,7,8,11.5).$ This type of inputs
do not trigger the formulaic bias, hence do not cause any nonextant data points. An instance of
partially synchronous inputs is  $\Pi^{(1)} = (2,3,4,5,7,8,11.5),$ and $\Pi^{(2)} = (1,3,5,6,7,9,12).$

\subsubsection{Asynchronicity and boundary effects} 

  For the asynchronous setting,
  the minimal possible initial and terminal configuations for $\Pi^{(2)}_j$ are as follows.
  Minimal in the sense that a configuration captures with minimal amount of data points 
  when {\it nonextant data} points can arise due to asynchronous boundary condition.
  In our case, determining whether $t^{(2)}_1$ and $t^{(2)}_{M^{(2)}-1}$ are {\it nonextant}.
  Further, we label points in $\Pi^{(1)}$ with $A$ and points in $\Pi^{(2)}$ with $B$ and
  observe the input sequences in the following configurations.
\tikzset{
  fake double/.style 2 args={
    label={[#1, minimum size={#2}]center:}}}
    \begin{figure}[!htb]
        \centering
           \begin{subfigure}[b]{0.4\textwidth}
             \centering
             \caption{Initial Configuration: $BABBA$}
\begin{tikzpicture}[
  o/.style={shape=circle, draw, inner sep=+0pt, minimum size=+2mm},
  O/.style={o,       fake double={o}{+2.6mm}},
  c/.style={o, fill, draw=none},
  C/.style={o, fill, fake double={o}{+3mm}},
  >=Latex,
  line direction/.style={rotate=0, shift={(90:#1)}}]
\foreach \i/\t in { 0/\Pi^{(2)}, 1/\Pi^{(1)}}
  \draw[|-,line direction=\i*.65] 
        (0, 5) coordinate (st\i) node[ left]{$\t$}
        -- coordinate[pos=-.1] (mk\i) +(right:5) coordinate (en\i);
\node[below right, align=center] at (en0) {};
\path (st1) -- (en1)
  node foreach \Pos in  {.25, .8} [c,pos=\Pos]{};
\path (st0) -- (en0)
node foreach \Pos in {0.1,.65} [c, pos=\Pos]{}
node foreach \Pos in {0.43} [C, pos=\Pos]{} ;
\end{tikzpicture}
             \label{subfig:configa}
         \end{subfigure}
         \hfill
         \begin{subfigure}[b]{0.4\textwidth}
             \centering
             \caption{Terminal Configuration:$ABBAB$}
\begin{tikzpicture}[
  o/.style={shape=circle, draw, inner sep=+0pt, minimum size=+2mm},
  O/.style={o,       fake double={o}{+2.6mm}},
  c/.style={o, fill, draw=none},
  C/.style={o, fill, fake double={o}{+3mm}},
  >=Latex,
  line direction/.style={rotate=0, shift={(90:#1)}}]
\foreach \i/\t in { 0/\Pi^{(2)}, 1/\Pi^{(1)}}
  \draw[-|,line direction=\i*.65]
        (0, 5) coordinate (st\i) node[ left]{$\t$}
    -- coordinate[pos=-.1] (mk\i) +(right:5) coordinate (en\i);
\node[below right, align=center] at (en0) {};
\path (st1) -- (en1)
  node foreach \Pos in  {.25, .65} [c,pos=\Pos]{};
\path (st0) -- (en0)
node foreach \Pos in {0.4, .9} [c, pos=\Pos]{}
node foreach \Pos in {0.55} [C, pos=\Pos]{};

\end{tikzpicture}
             \label{subfig:configb}
         \end{subfigure}
         \hfill
         \begin{subfigure}[b]{0.4\textwidth}
             \centering
             \caption{Initial Configuration:$BABAB$}
\begin{tikzpicture}[
  o/.style={shape=circle, draw, inner sep=+0pt, minimum size=+2mm},
  O/.style={o,       fake double={o}{+2.6mm}},
  c/.style={o, fill, draw=none},
  C/.style={o, fill, fake double={o}{+3mm}},
  >=Latex,
  line direction/.style={rotate=0, shift={(90:#1)}}]
\foreach \i/\t in { 0/\Pi^{(2)}, 1/\Pi^{(1)}}
  \draw[|-,line direction=\i*.65]
        (0, 5) coordinate (st\i) node[ left]{$\t$}
    -- coordinate[pos=-.1] (mk\i) +(right:5) coordinate (en\i);
\node[below right, align=center] at (en0) {};
\path (st1) -- (en1)
  node foreach \Pos in  {.25, .5,.85} [c,pos=\Pos]{};
\path (st0) -- (en0)
  node foreach \Pos in {0.1, .35, .65} [c, pos=\Pos]{};
\end{tikzpicture}
       \label{subfig:configc}
     \end{subfigure}
   \hfill
         \begin{subfigure}[b]{0.4\textwidth}
             \centering
             \caption{Terminal Configuration:$ABABAB$}
\begin{tikzpicture}[
  o/.style={shape=circle, draw, inner sep=+0pt, minimum size=+2mm},
  O/.style={o,       fake double={o}{+2.6mm}},
  c/.style={o, fill, draw=none},
  C/.style={o, fill, fake double={o}{+3mm}},
  >=Latex,
  line direction/.style={rotate=0, shift={(90:#1)}}]
\foreach \i/\t in { 0/\Pi^{(2)}, 1/\Pi^{(1)}}
  \draw[-|,line direction=\i*.65]
        (0, 5) coordinate (st\i) node[ left]{$\t$}
    -- coordinate[pos=-.1] (mk\i) +(right:5) coordinate (en\i);
    \node[below right, align=center] at (en0) {};
    \path (st1) -- (en1)
  node foreach \Pos in  {0.25,.5, .8} [c,pos=\Pos]{};
\path (st0) -- (en0)
node foreach \Pos in {.35, .65, 0.9} [c, pos=\Pos]{};
\end{tikzpicture}
             \label{subfig:configd}
           \end{subfigure}
           \caption{Minimal initial and terminal configuration for $\Pi^{(2)}$ in asynchronous setting.
  Note that points plotted as \!\!\nonextantP{} represent \textit{nonextant data points}, and will be cancelled during the computation,
  whereas other points survive in the computation.}
\label{fig:four_configurations}
\end{figure}
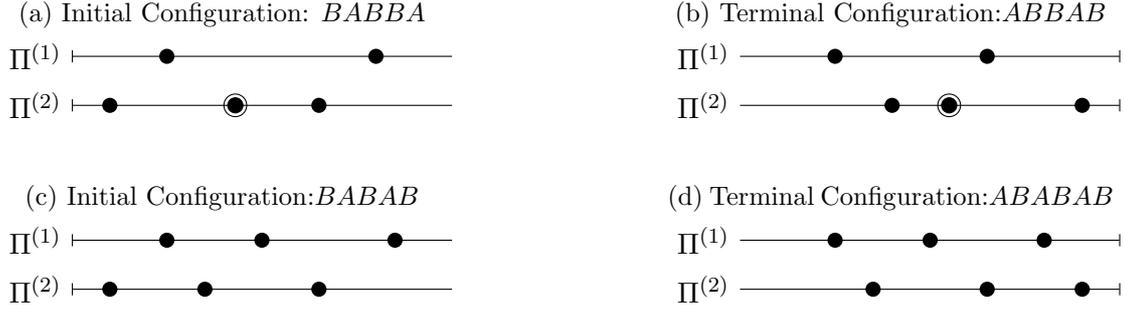

We observe a combination of configurations~\ref{subfig:configa} and ~\ref{subfig:configb},
that yields the following sequence $BABBAB$ in the closed observation time interval, which contains subsequence $BABBA$ and subsequence $ABBAB$.

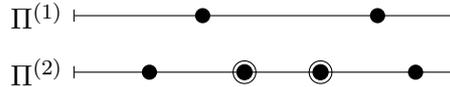
\begin{figure}[!htpb]
  \centering
\begin{tikzpicture}[
  o/.style={shape=circle, draw, inner sep=+0pt, minimum size=+2mm},
  O/.style={o,       fake double={o}{+2.6mm}},
  c/.style={o, fill, draw=none},
  C/.style={o, fill, fake double={o}{+3mm}},
  >=Latex,
  line direction/.style={rotate=0, shift={(90:#1)}}]
\foreach \i/\t in { 0/\Pi^{(2)}, 1/\Pi^{(1)}}
  \draw[|-|,line direction=\i*.75]
        (0, 5) coordinate (st\i) node[ left]{$\t$}
    -- coordinate[pos=-.1] (mk\i) +(right:5) coordinate (en\i);
\node[below right, align=center] at (en0) {};
\path (st1) -- (en1)
  node foreach \Pos in  {.34, .8} [c,pos=\Pos]{};
\path (st0) -- (en0)
node foreach \Pos in {0.2, .9} [c, pos=\Pos]{}
node foreach \Pos in { .45, .65} [C, pos=\Pos]{};
\end{tikzpicture}
\caption{Combination of configurations~\ref{subfig:configa} and ~\ref{subfig:configb} in closed interval, forming sequence  $BABBAB$,
which contains subsequences  $BABBA$ and $ABBAB.$ This configuration yields to consecutive \textit{nonextant data points} in $\Pi^{(2)}.$ }
\end{figure}

\subsubsection{Asynchronicity and absence of boundary effects}

For the remaining  $t^{(2)}_j$ in $\Pi^{(2)}$ that are not impacted by boundary effects,
the following minimal configuration induces \textit{nonextant data points}, giving
rise to sequence $BBB$ in this case. Analogously, $AAA$ for $t^{(1)}_i$ in $\Pi^{(1)}.$

\begin{figure}[!htp]
  \centering
\begin{tikzpicture}[
  o/.style={shape=circle, draw, inner sep=+0pt, minimum size=+2mm},
  O/.style={o,       fake double={o}{+2.6mm}},
  c/.style={o, fill, draw=none},
  C/.style={o, fill, fake double={o}{+3mm}},
  >=Latex,
  line direction/.style={rotate=0, shift={(90:#1)}}]
\foreach \i/\t in { 0/\Pi^{(2)}, 1/\Pi^{(1)}}
  \draw[line direction=\i*.75]
        (0, 5) coordinate (st\i) node[ left]{$\t$}
    -- coordinate[pos=-.1] (mk\i) +(right:5) coordinate (en\i);
\node[below right, align=center] at (en0) {};
\path (st1) -- (en1)
  node foreach \Pos in  {.1, .9} [c,pos=\Pos]{};
\path (st0) -- (en0)
  node foreach \Pos in {0.3,.8 } [c, pos=\Pos]{}
   node foreach \Pos in { .6} [C, pos=\Pos]{};
\end{tikzpicture}
\caption{ For $t^{(2)}_j$ in $\Pi^2$ not affected by boundary condition. }
\end{figure}
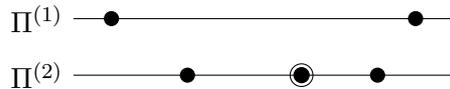

\subsection{Overlaps}\label{overlaps}

We let $m$ be the {\it maximal number of overlaps} between $I^{(1)}$ and $I^{(2)}.$ 

\begin{lem}\label{lemma1}
  Label points in $\Pi^{(1)}$ with $A$ and points in $\Pi^{(2)}$ with $B$. 
  Then the number of overlaps, $m$, is ``almost'' equal to the total number of points.
  So asymptotically, for large $T$, this is $T(a+b)$. 
\end{lem}

\begin{proof} 
Add points of both types at $0$ and $T$; this will add one interval of
each type at each end, and, in expectation, $O(1)$ intersections. (It should
be easy to calculate this error precisely, at least for large $T$.)
If we have double points at $0$ and $T$, then the number of intersections
equals 1 plus the number of interior points, by induction: with no interior points we
have just two intervals and one overlap; if you add one point, it will
divide one existing interval into two, and the total number of overlaps of
these with intervals of the opposite class will increase by 1. (We assume
that with probability $1$, all interior points are distinct.)
\end{proof}

\begin{remark}\label{remark:points:overlaps}
  Note that inputs generated by our {\it asynchronous adversary} are completely asynchronous.
  Additionally, we can assure that the number of intersections or overlaps,$m$,
  equals the number of data points in $\Pi^{(1)}$ and $\Pi^{(2)}$ minus three, or,
  $m=(a+b)-3.$ Reviewing our concrete example~\ref{example:concrete}, helps to clarify this. 
\end{remark}

\subsection{Telescoping, nonextant data points, average data point loss}\label{nonextant:pts}

While there are various ways to formalize \textit{nonextant data points}, i.e., the
output of the estimator does not depend on the value of the functions at those points,
or, points that completely cancel out during the computation without influencing the
output of the estimator, a mathematically precise definition appears to be
estimator specific. For the (\HY)-estimator, we note necessary and
sufficient conditions to trigger the formulaic bias which results from an
interaction of telescoping property and indicator function, is input specific.
Completely synchronous input does not induce this bias; asynchronous input is necessary.
One way to obtain a clearer mathematical definition is by analyzing the summation formula,~\ref{our:def:form}.
A key insight to the formulaic bias, follows from considering what happens
when we fix one of the index variables while summing over the other, as seen in our concrete example~\ref{example:concrete}.
In particular, observe which endpoints get cancelled due to telescoping while others survive when the
indicator function, $\mathds{1}_{\left\{ \l(t_{i-1}^{(1)}, t_{i}^{(1)}] \cap \l(t_{j-1}^{(2)},t_{j}^{(2)}]\neq\emptyset\right\}},$ holds true.
Using this insight, we can formalize the {\it nonextant intervals}, intervals that might contain {\it nonextant points}.
Following Remark~\ref{remark:observation:times}, for simplicity, we focus on observation times, rather than prices.

\begin{prop}
  Let $\Pi^{(1)} := (t_i^{(1)})_{i=0,1,..,M^{(1)}}$ and $\Pi^{(2)} := (t_j^{(2)})_{j=0,1,..,M^{(2)}}$,  be the inputs to the (\text{\HY})-estimator,
  and let $t_{i^+}^{(2)} := \min\{ t_v^{(2)} : t_v^{(2)} > t_i^{(1)}\}$ and  $t_{i^-}^{(2)} := \max\{ t_v^{(2)} : t_v^{(2)} < t_i^{(1)}\}$, then given endpoints,
  $t_{i-1}^{(1)}$ and $t_{i}^{(1)}$ in $\Pi^{(1)}$ for $1<i<M^{(1)}$,  nonextant data points are precisely the data points in $\Pi^{(2)}$
  that occur between $(t_{(i-1)^+}^{(2)} .. t_{(i)^-}^{(2)}).$
\end{prop}

\begin{proof}
  Note that $\l(t_{i-1}^{(1)}, t_{i}^{(1)}] \bigcap \l(t_{j-1}^{(2)},t_{j}^{(2)}]\neq\emptyset$ if and only if $t_i^{(1)}>t^{(2)}_{j-1}$ and $t^{(1)}_{i-1}<t^{(2)}_j$.
  This is equivalent to $t^{(2)}_{j-1}\le t^{(2)}_{i^-}$ and $t^{(2)}_j \ge t^{(2)}_{(i-1)^+}$. Hence, for the actual inputs with prices,
  the terms  $\Delta P^{(1)}_{i} \cdot \Delta P^{(2)}_j$  for a fixed $i$ are those with such $t^{(2)}_j.$ Then, telescoping the sum
  of them cancels the terms with $t^{(2)}_j$ for $t^{(2)}_{j}\le t^{(2)}_{i^-}$ and  $t^{(2)}_j \ge t^{(2)}_{(i-1)^+}.$  However, these are
  not necessarily the nonextant ones, since the smallest of them will appear in a term with $\Delta P^{(1)}_{i-1}$ (unless $i=1$), and
  similarly for the largest.
\end{proof}

\begin{remark}
  Given asynchronous inputs, separate consideration for the initial case when $i=1$, and terminal case when $i=M^{(1)},$ is needed.
  Let $i=1$, then, given endpoints $t_{0}^{(1)}$ and $t_{1}^{(1)}$ in $\Pi^{(1)},$ the nonextant interval is simply  $\left(\min\{ t_{(0)^-}^{(2)}, t_{(0)+}^{(2)}\} .. t_{(1)^-}^{(2)}\right).$
  We note since endpoint $\min\{ t_{(0)^-}^{(2)}, t_{(0)+}^{(2)}\}$, which is the inital endpoint of the intial overlap, and hence, not a nonextant data point, can
  occur prior to $t_0^{(1)}$, therefore, we can no longer simply use $t_{(0)^+}^{(2)}.$  Analogous consideration for the terminal case, when $i=M^{(1)}.$  
\end{remark}

Let us now characterize {\it nonextant data points} in a more succinct and algorithmically
conducive way. Our procedure decides whether a data point is {\it nonextant} based
on whether the union of its (two) immediate adjacent
intervals forms a subset over the corresponding overlapping interval. For boundary
effects potentially impacting second and penultimate data point, an additional
condition is employed.

We abbreviate the cumulative number of {\it nonextant data points}
in both $\Pi^{(1)}$ and $\Pi^{(2)}$ as $f.$
With regard to quantifying and defining {\it average data loss}, we observe
that, asymptotically, for large $T$, our metric converges to the
theoretical values (as evidenced in Table 1 of section~(\ref{eda})).
A natural definition for {\it average data loss}, agnostic
to how inputs are generated, can be based on overlaps.  
Since overlaps form a necessary but not sufficient condition for nonextant
data points, defining data loss as the ratio of the number of total nonextant data points
over the number of total overlaps might seem reasonable,i.e., $f/m.$
As we noted in the above Remark~\ref{remark:points:overlaps}, we can compute
the number of overlaps from the number of data points.

In this note we specify the underlying stochastic process that our inputs
are generated from and {\it estimate the average proportion of data loss}
as shown in~(\ref{average:proportion:nonextant}) and~(\ref{corollary1}).

\begin{thmx}\label{theo8} Let $\Pi^{(1)} := (t_i^{(1)})_{i=0,1,..,M^{(1)}}$ and $\Pi^{(2)} := (t_j^{(2)})_{j=0,1,..,M^{(2)}}$,
  be the inputs to the (\text{\HY})-estimator. Then, $t^{(2)}_j$, for $1<j<M^{(2)}-1$,
  is {\it nonextant} if and only if $I^{(2)}_j\cup I^{(2)}_{j+1} \subseteq I^{(1)}_i$ for some $i.$
\end{thmx}

\begin{proof}
  $t^{(2)}_j$ is {\it nonextant}, if and only if, for
  exactly one $i$ in $I^{(1)}_i$, $I^{(1)}_i\cap I^{(2)}_j\neq\emptyset \iff I^{(1)}_i\cap I^{(2)}_{j+1}\neq\emptyset$
  because this will make all terms with $t^{(2)}_j$ (or, $P^{(2)}_{t_j}$, if we had used $P_{\Pi^{(1)}}^{(1)}, P_{\Pi^{(2)}}^{(2)}$ as inputs) cancel.
  This, in turn, happens, if and only if, $I^{(2)}_j\cup I^{(2)}_{j+1} \subseteq I^{(1)}_i$ for
$1<j<M^{(2)}-1$. For $j=1$, and $j=M^{(2)}-1,$ if the condition $I^{(2)}_j\cup I^{(2)}_{j+1} \subseteq I^{(1)}_i$
does not hold, we need test whether $I^{(2)}_j\cup I^{(2)}_{j+1}$ intersects  $I^{(1)}_i$ for exactly one $i.$\footnote{%
  Perhaps alternatively,
  $t^{(2)}_j$ is {\it nonextant}, if and only if,
  $\left(I^{(2)}_j\cup I^{(2)}_{j+1} \right) \cap I^{(1)}_i \neq \emptyset \,
  \text{and}\, \forall_{k \neq i} \left( (I_{j}^{(2)} \cup I_{(j+1)}^{(2)}) \cap I_{k}^{(1)} = \emptyset \right).$
}%
\end{proof}

With the above out of the way, we address the
questions of what the minimal average data point loss is given
inputs $\Pi^{(1)}$ and $\Pi^{(2)}$ generated by our $(a,b)$-asynchronous adversary,
and for what values of $a$ and $b$ can we guarantee this.

\begin{thmx}\label{average:proportion:nonextant}
  Assume that  $\Pi^{(1)}$ and  $\Pi^{(2)}$ are generated by our $(a,b)$-asynchronous adversary,
  and are inputs to the (\text{\HY})-estimator. Additionally, label points in $\Pi^{(1)}$ with $A$
  and points in $\Pi^{(2)}$ with $B$. Further, let $T$ be large, and ignore border effects. 
  Then the proportion of all points, i.e., $\Pi^{(1)}\cup\Pi^{(2)},$ that are in $\Pi^{(1)}$ and are nonextant there\footnote{NB: The expected proportion of points that are in $\Pi^{(1)}$ is $a/(a+b)$, and thus the expected proportion of nonextant points in $\Pi^{(1)}$ is $(a/(a+b))^2$.},
  is equivalent to the expected proportion of occurrence of substring $AAA$ which is $(a/(a+b))^3$.
\end{thmx}

\begin{proof}
 Using language of and construction in Lemma \ref{lemma1},
 a point labelled A is nonextant $\iff$ it is preceded and followed by A.
In that case, it borders two intervals that are subsets of the same class 2-interval, but if not,
then, say the point is followed by B, and thus the A-interval to the right is the first A-interval
that intersects the next B-interval, and there the point is not nonextant.
Hence, ignoring border effects of the order O(1), the number of nonextant points should be the same
as the number of substrings $AAA$. The number of triples equals the number of
total data points minus two which is approximately equal to the number of overlaps, $m$.
Each point has probability $a/(a+b)$ to be an $A$; thus the expected proportion of $AAA$ is $(a/(a+b))^3.$
\end{proof}

\begin{cor}\label{corollary1}
  Let  $\Pi^{(1)}$ and  $\Pi^{(2)}$ be inputs to the (\HY)-estimator, which were generated by
  our $(a,b)$-asynchronous adversary. Further, assume $T$ is large, and ignore border
  effects. 
  Then the (cumulative) average proportion of nonextant data points in both
  $\Pi^{(1)}$ and $\Pi^{(2)}$  is provided by
  $f(a,b) = (\frac{a}{a+b})^3 + (\frac{b}{a+b})^3,$ where $a>0$ and $b>0.$
\end{cor}

  \begin{thmx}
    Let $f(a,b) = (\frac{a}{a+b})^3 + (\frac{b}{a+b})^3,$ where $a>0, b>0.$ Then $f(a,b)$
    attains the global minimum value $f(a,b) = 1/4$ for equal rates, $a=b$. 
    \end{thmx}

    \begin{proof}
      Given $a>0$ and $b>0$, we easily obtain $f(a,b)>0.$  After algebraic manipulation, we can rewrite
      $f(a,b) = \frac{\left(a^2-ab+b^2\right)}{(a+b)^2} = 1 -\frac{3ab}{\left(a+b\right)^2}.$
      Now, express $f(a,b)$ as $g(t)$ via the following substitution $t=\frac{a}{a+b}$;
      hence, $(1-t) = \frac{b}{a+b}.$ We obtain $g(t) = 1- 3t\left(1-t\right).$
      Further, after completing the square, we obtain $g(t) = 3t^2-3t+1,$
      which is a parabola with positive coefficient of $t^2.$ Therefore,
      its vertex represents the global minimum at $t_{\text{vertex}} = \frac{1}{2}.$
      Further, this value implies that the global minimum is attained
      when $a = b,$ since $t=1/2=\frac{a}{(a+b)}$ and $1-t=1-1/2=\frac{b}{(a+b)}.$
      Evaluating $g(t)$ at $t_{\text{vertex}}$, we obtain the global minimum value of $f(a,b)$
      via $g(\frac{1}{2}) = 1- 3\left(\frac{1}{2}\right)\left(1-\frac{1}{2}\right)= \frac{1}{4}.$  
    \end{proof}

\subsection{Algorithms}
This section includes the relevant material for the algorithms. We outline
the implementation for counting the exact number of nonextant data points,which is based on
proof of theorem~\ref{theo8}. The next section proposes a mergesort algorithm.

\subsubsection{Algorithm counting nonextant data points}

The algorithm~\ref{algo1} computes nonextant data points excluding boundary points,
$t^{(2)}_1$ and $t^{(2)}_{M^{(2)} - 1}.$ To determine whether these two points are
\textit{nonextant}, we have for $j=1$,

 $\left(I^{(2)}_1\cup I^{(2)}_{2} \right) \cap I^{(1)}_i \neq \emptyset \,
 \text{and}\, \forall_{k \neq i} \left( (I_{1}^{(2)} \cup I_{2}^{(2)}) \cap I_{k}^{(1)} = \emptyset \right),$ for
 some $i.$ Analogous consideration for $j=M^{(2)}-1.$

\newcommand{\mycomment}[1]{\hfill\textcolor{gray}{\textit{// #1}}}

\begin{algorithm}[H]
\caption{Counting Nonextant Data Points in $\Pi^{(2)}$ given inputs: $\Pi^{(1)}$ and $\Pi^{(2)}$}
\label{algo1}
\begin{algorithmic}[1]
\Procedure{CountNonextantDataPoints}{$\Pi^{(2)}, \Pi^{(1)}$}
\State $c \gets 0 $ \mycomment{initialize counter for nonextant data points in $\Pi^{(2)}$}
\State $m \gets length.(\Pi^{(1)})-1$ \mycomment{\# of intervals in $I^{(1)}_i$; $\Pi^{(1)}$ indexed from $0$  to $M^{(1)}$}
\State $n \gets length.(\Pi^{(2)})-1$ \mycomment{\# of intervals in $I^{(2)}_j$; $\Pi^{(2)}$ indexed from $0$  to $M^{(2)}$}
    \For{$j \gets 2$ \textbf{to} $n-2$}
        \State $I_{j}^{(2)} \gets$ $\left(\Pi_{j-1}^{(2)}, \Pi_{j}^{(2)}\right]$
        \State $I_{(j+1)}^{(2)} \gets \left(\Pi_{j}^{(2)}, \Pi_{j+1}^{(2)}\right]$
        \For{$i \gets 1$ \textbf{to} $m$}
            \State $I_{i}^{(1)} \gets \left(\Pi_{i-1}^{(1)}, \Pi_{i}^{(1)}\right]$
            \If{$\left(I_{j}^{(2)} \cup I_{(j+1)}^{(2)}\right) \subseteq I_{i}^{(1)}$}
                \State $c \gets c + 1$ 
            \EndIf
        \EndFor
        \EndFor
    \State \textbf{return} c
\EndProcedure
\end{algorithmic}
\end{algorithm}

\subsubsection{Counting nonextant data points via mergesorting observation times}

This algorithm harnesses the property that inputs need to possess to generate {\it nonextant data points}.
It is equivalent to counting the number of subsequences `AAA' or `BBB' (apart from initial boundary conditions, `ABAA' or `BABB',
and terminal boundary condition, `AABA' and `BBAB') while mergesorting $\Pi^{(1)}$ and $\Pi^{(2)}$,
whose inputs are naturally sorted in ascending order and have accordingly been labelled with $A$
(for inputs in $\Pi^{(1)}$) and $B$ (for inputs in $\Pi^{(2)}$),
respectively. This can be done in $\bigO(n)$ time, where $n=|\Pi^{(1)}| + |\Pi^{(2)}|.$

\subsubsection{Counting nonextant data points given that entire observation sequence is provided}

Alternatively, if we are provided with the entire input sequence in terms of A's and B's, then counting nonextant data
points appears to be a routine application of the Knuth-Morris-Pratt or KMP algorithm~\cite{KMP-1977}.

\subsubsection{Simulation Algorithm for (a,b)-asynchronous adversary}\label{simulation:algo}

In order to generate inputs via our $(a,b)$-asynchronous adversary, we need to
implement a simple algorithm that generates data points for observation times from 
a homogeneous Poisson process with intensity or rate $a$. Our implementation 
harnesses the fact that times between successive events for a homogeneous Poisson 
process are independent exponential random variables each with rate $a.$
We devise a variant of Ross's implementation for generating the first $T$ time
units of two independent homogenous Poisson processes with rates $a$ and $b,$ respectively,
as illustrated in~\cite[p.~72]{R-1990}.

\subsection{Computational Exploration}\label{eda} 

We employ the algorithm in~\ref{simulation:algo} to generate observation times for 
$\Pi^{(1)}$ and $\Pi^{(2)}$ for two independent homogenous Poisson processes 
with varying rates, $r.$
Table~1 shows computed values for the {\it data loss}, which we have
previously defined to be the ratio of the cumulative, ($\Pi^{(1)}$ and $\Pi^{(2)}$),
{\it number of nonextant points} over the {\it number of overlaps}. All entries are provided in the form
$\mu \pm \sigma,$ where $\mu$ is the sample mean and $\sigma$ is an estimate
of the standard deviation. The number of test runs for each instance or experiment
was 1000.

\bigskip
\centerline{Table 1.\quad Empirical data point loss, {\it f}/{\it m},}
\centerline{for inputs $\Pi^{(1)}$ and $\Pi^{(2)}$ generated by}
\centerline{$(a,b)$-asynchronous adversary over $\l(0,T]$,}
\centerline{compared to theoretical values via $f(a,b)$}
\centerline{(excluding boundary points.)}
$$\vcenter{\halign{$#$\hfil\quad
&$#$\hfil\quad&$#$\hfil\quad&$#$\hfil\quad&$#$\hfil\quad&$#$\hfil\cr
&\hfil r=1,1&\hfil r=1,{1\over2}&\hfil r=1,{1\over4}&\hfil r=1, {1\over 10}
&\cr
\noalign{\smallskip}
T=100&\phantom{.}0.249\pm0.040&\phantom{.}0.330\pm0.055&\phantom{.}0.512\pm0.072& 0.742\pm0.074 &\cr%
T=1000&\phantom{.}0.250\pm0.012 &\phantom{.} 0.333\pm0.018&\phantom{.}0.519\pm0.023 %
&0.751\pm 0.020&\cr
T=10000&\phantom{.}0.250\pm0.004&\phantom{.}0.333\pm0.006&\phantom{.}0.520 \pm 0.007 &0.752\pm0.007&\cr
T=100000&\phantom{.}0.250\pm0.001&\phantom{.}0.333\pm 0.002&\phantom{.}0.520\pm0.002& 0.752\pm 0.002&\cr
\text{Theoretical:}&f(1,1)=1/4&f(1,\frac{1}{2})=1/3&f(1,\frac{1}{4})=\frac{13}{25}& f(1,\frac{1}{10})=\frac{91}{121} &\cr
}}$$

\section*{Acknowledgement}
I heartily thank Svante Janson for extensive comments, insightful feedback,
assistance in formalizing the proposition as well as correcting errors in earlier drafts,
and I am greatly indebted to Daniel M. Kane for helping narrow the initial
scope of the problem statement, sharpening the precision of definitions and proposition,
as well as extensive comments that helped enhance the clarity and readability
of previous versions.


\appendix

\section*{Appendix}

\subsection{Computation I}\label{appendix:1}
\begin{footnotesize}
\begin{align}
{\left<P_{\Pi^{(1)}}^{(1)}, P_{\Pi^{(2)}}^{(2)}\right>}_{{\text{\HY}}}&:= \sum_{j=1}^{5} \sum_{i=1}^{6} \left( P_{t_i^{(1)}}^{(1)} -  P_{t_{i-1}^{(1)}}^{(1)} \right) \left(P_{t_j^{(2)}}^{(2)} - P_{t_{j-1}^{(2)}}^{(2)}\right)\mathds{1}_{\left\{ \l(t_{i-1}^{(1)},t_{i}^{(1)} ] \bigcap \l(t_{j-1}^{(2)}, t_{j}^{(2)}]\neq\emptyset \right\}} \label{async2:appendix} \\
&\, \text{factoring sequentially, via $(P_{t_4^{(1)}}^{(1)} - P_{t_{3}^{(1)}}^{(1)})$, and pairwise encircling cancellations}  \nonumber \\
&\, = \left(P_{t_1^{(2)}}^{(2)} - P_{t_{0}^{(2)}}^{(2)}\right) \left[ \left(\text{\encircle{$P_{t_1^{(1)}}^{(1)}$}} -  P_{t_{0}^{(1)}}^{(1)} \right) + \left(\text{\encircled{$P_{t_2^{(1)}}^{(1)}$}} -  \text{\encircle{$P_{t_{1}^{(1)}}^{(1)}$}}\right) + \left(P_{t_3^{(1)}}^{(1)} - \text{\encircled{$P_{t_{2}^{(1)}}^{(1)}$}}\right)  \right] \nonumber \\
&\, + \left( P_{t_4^{(1)}}^{(1)} -  P_{t_{3}^{(1)}}^{(1)} \right) \left[ \left(\text{\encircle{$P_{t_1^{(2)}}^{(2)}$}} - P_{t_{0}^{(2)}}^{(2)}\right)  + \left(P_{t_2^{(2)}}^{(2)} - \text{\encircle{$P_{t_{1}^{(2)}}^{(2)}$}}\right)  \right] \nonumber \\
&\, + \left(P_{t_2^{(2)}}^{(2)} - P_{t_{1}^{(2)}}^{(2)}\right) \left[  \left( \text{\encircle{$P_{t_5^{(1)}}^{(1)}$}} -  P_{t_{4}^{(1)}}^{(1)}\right) + \left(P_{t_6^{(1)}}^{(1)} -  \text{\encircle{$P_{t_{5}^{(1)}}^{(1)}$}}\right)  \right] \nonumber \\
&\, + \left( P_{t_6^{(1)}}^{(1)} -  P_{t_{5}^{(1)}}^{(1)}\right) \left[  \left(\text{\encircle{$P_{t_3^{(2)}}^{(2)}$}} - P_{t_{2}^{(2)}}^{(2)}\right) + \left(\text{\encircled{$P_{t_4^{(2)}}^{(2)}$}} - \text{\encircle{$P_{t_{3}^{(2)}}^{(2)}$}}\right) + \left(P_{t_5^{(2)}}^{(2)} - \text{\encircled{$P_{t_{4}^{(2)}}^{(2)}$}}\right)\right] \nonumber \\
\end{align}
\end{footnotesize}

Further tidying up, we obtain,
\begin{tiny}
\begin{equation}\label{new}
\left(P_{t_1^{(2)}}^{(2)} - P_{t_{0}^{(2)}}^{(2)}\right) \left( P_{t_3^{(1)}}^{(1)}  - P_{t_{0}^{(1)}}^{(1)}  \right) + \left( P_{t_4^{(1)}}^{(1)} -  P_{t_{3}^{(1)}}^{(1)} \right) \left( P_{t_2^{(2)}}^{(2)} -  P_{t_{0}^{(2)}}^{(2)} \right)  + \left(P_{t_2^{(2)}}^{(2)} - P_{t_{1}^{(2)}}^{(2)}\right) \left(P_{t_6^{(1)}}^{(1)} - P_{t_{4}^{(1)}}^{(1)}\right) + \left( P_{t_6^{(1)}}^{(1)} -  P_{t_{5}^{(1)}}^{(1)}\right)  \left(  P_{t_{5}^{(2)}}^{(2)} -  P_{t_{2}^{(2)}}^{(2)}\right)
\end{equation}
Via this sequence of computation we note four cancellations or nonextant data points.
However, there are 4 product summands.
Obvious algebraic manipulation shows that we can derive (\ref{old}).
\end{tiny}

\end{document}